\documentclass[conference]{IEEEtran}
\IEEEoverridecommandlockouts
\usepackage{cite}
\usepackage{amsmath,amssymb,amsfonts,amsthm}
\usepackage{graphicx}
\usepackage{textcomp}
\usepackage{xcolor}
\usepackage{bbm}
\usepackage{dsfont}
\usepackage{comment}

\usepackage{algorithm}
\usepackage[noend]{algpseudocode}

\usepackage{float}

\newtheorem{assumption}{Assumption}
\newtheorem{proposition}{Proposition}

\def\BibTeX{{\rm B\kern-.05em{\sc i\kern-.025em b}\kern-.08em
    T\kern-.1667em\lower.7ex\hbox{E}\kern-.125emX}}

\begin{document}

\title{ATOM: Abstract Transition and Option Modeling - Experience-Informed Hierarchical Navigation \\
}

\author{\IEEEauthorblockN{Changxuan Yao and Panagiotis Tsiotras}
\IEEEauthorblockA{Institute for Robotics and Intelligent Machines\\ 
Georgia Institute of Technology,
Atlanta, GA \\
\{cyao48, tsiotras\}@gatech.edu}
}

\maketitle

\begin{abstract}
We propose ATOM (Abstract Transition and Option Modeling), a hierarchical navigation framework that combines higher-level graph planning with reusable lower-level option policies for efficient repeated navigation in large-scale 3D environments represented by hierarchical structures such as 3D scene graphs (3DSGs). 
ATOM operates on a hierarchical graph representation, implemented using, for instance, the \textsf{Rooms} and \textsf{Places} layers of a 3DSG, where each directed abstract transition is associated with a navigation option. 
For each transition, ATOM computes and stores a reusable option policy along with its state-dependent cost-to-go. 
Across repeated navigation tasks, ATOM maintains a smoothed empirical distribution over option initiation states and uses it to progressively refine the corresponding abstract transition cost. 
Ablation studies show that ATOM reduces cumulative repeated-planning time by 97.5\% relative to a monolithic MDP baseline, while maintaining trajectories within approximately 7\% of a Global Places A* shortest-path reference, while refinement reduces abstract transition-cost estimation error by 10.3\%.
Physical experiments further validate ATOM on an F1Tenth vehicle, successfully completing all five navigation queries with 43\% of
inter-region option executions reusing previously generated solutions.
\end{abstract}


\section{Introduction}

Autonomous agents must navigate reliably in complex real-world environments, including self-driving vehicles \cite{auto-vehicles,self-driving}, delivery systems \cite{delivery-bots, delivery-2}, disaster response \cite{sar, disaster_response} and surveillance \cite{surv, monitoring}. 
Efficient navigation is essential in large-scale environments, where re-planning is required as new navigation tasks arise. Although classical graph search and MDP-based planning methods can compute paths and navigation policies, respectively, solving each new start-goal query independently does not leverage planning solutions from previous queries, potentially incurring unnecessary computation by repeatedly solving the same local navigation subproblems. In structured environments, many navigation queries traverse the same local connections, creating opportunities for reuse. This motivates a hierarchical approach in which local policies can be reused across start-goal queries rather than recomputed for each task.

\begin{figure}[t!]
\centerline{\includegraphics[scale=0.33]{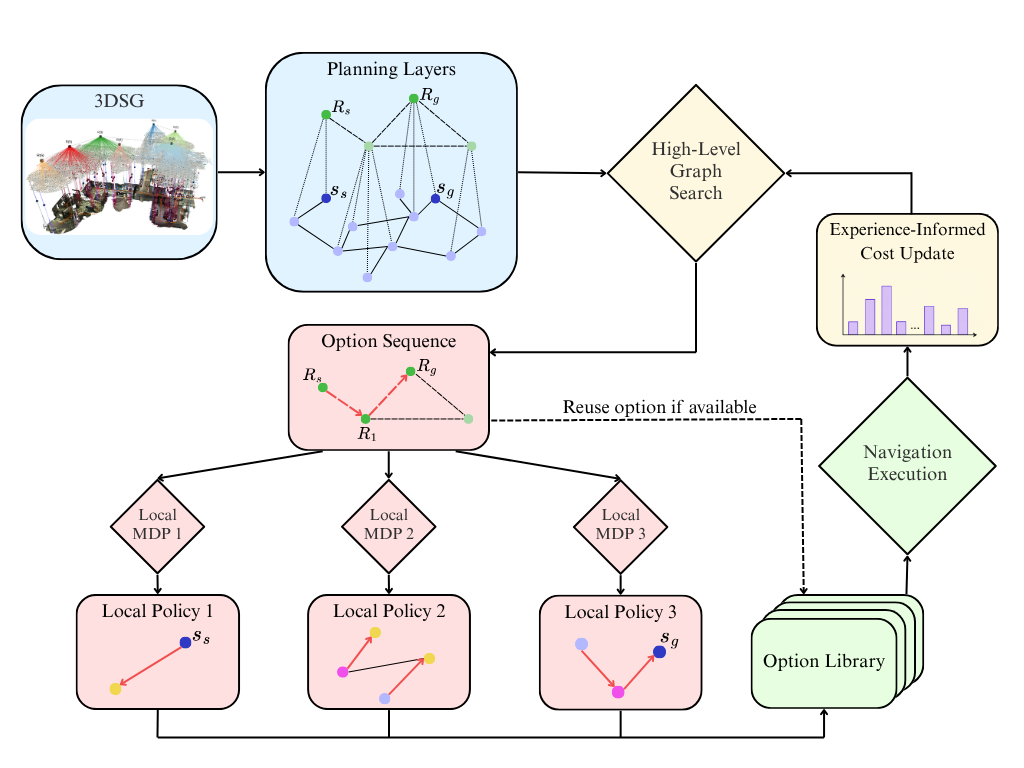}}
\caption{ATOM framework. Abstract transitions are represented by reusable options, while navigation experience refines transition costs for subsequent high-level planning.}
\label{framework}
\end{figure}

Hierarchical graph representations, such as 3D scene graphs (3DSGs), organize environments at multiple levels of abstraction and provide a natural foundation for hierarchical planning. Our objective is to take advantage of these hierarchical representations to identify recurring abstract transitions and reuse their local planning solutions across navigation tasks. We therefore propose ATOM, a hierarchical navigation framework implemented using the \textsf{Rooms} and \textsf{Places} layers of a 3DSG generated by Hydra \cite{hydra} for urban environments.
Hierarchical planning decomposes large planning problems into local subproblems \cite{multi-hier,hier_state_abs}, while options provide temporally extended actions for connecting decisions across abstraction levels \cite{sutton,scott,rodriguez}. An option consists of an initiation set, a policy, and a termination condition, defining where it can begin, how it executes, and when it ends \cite{sutton}. ATOM associates each directed abstract transition with a reusable navigation option whose policy operates on the lower-level graph.

As illustrated in Figure~\ref{framework}, ATOM uses graph search (e.g., A*) to select a sequence of abstract transitions. For each such transition, a local MDP is solved once to obtain the corresponding policy and state-dependent cost-to-go, which are stored for reuse. As options are executed across navigation tasks, their smoothed empirical initiation-state distributions reweight the stored state-dependent cost-to-go, refining abstract transition costs for subsequent higher-level planning.

The main contributions of this work are:

\begin{itemize}
\item A hierarchical navigation framework (ATOM) that represents directed abstract transitions as reusable navigation options, avoiding repeated local policy generation.
\item An experience-informed method that refines abstract transition costs using smoothed empirical initiation-state distributions and stored state-dependent cost-to-go.
\item Theoretical analysis and simulation ablation studies together with physical-robot validation in a real indoor environment.
\end{itemize}

\section{Related Work}

\subsection{3D Scene Graphs}

3D scene graphs (3DSGs) represent environments as relational graphs containing geometric, semantic, and spatial information at multiple levels of abstraction \cite{sg_survey,3d_sg}. Armeni et al. \cite{3d_sg} introduced hierarchical 3DSGs for indoor environments, and Rosinol et al. \cite{kimera} extended this representation to dynamic scenes. Hydra \cite{hydra} enabled real-time construction and optimization of hierarchical 3DSGs, while CURB-SG \cite{curbsg} extended scene-graph representations to collaborative urban environments. Ray et al. \cite{ray} construct sparse task-and-motion planning domains from hierarchical 3DSGs, while SlideSLAM \cite{slide_slam} uses sparse metric-semantic representations for multi-robot navigation. 
ATOM, instead, uses an existing hierarchical graph as the planning abstraction. Each directed abstract transition is represented by a reusable option whose cost is refined from navigation experience. This enables lower-level planning solutions associated with recurring abstract transitions to be reused across navigation tasks, reducing repeated planning computation.

\subsection{Hierarchical Navigation}

Hierarchical planning is a well-studied field. 
It works by reducing planning complexity through abstract representations and lower-level subproblems \cite{abstrips,abstraction_theory}. 
Fernandez-Madrigal and Gonzalez \cite{multi-hier} applied this principle to multi-hierarchical graph search, recursively refining paths across abstraction levels. Hierarchical Path-Finding A* (HPA*) partitions maps into local clusters and precomputes crossing distances for higher-level search \cite{hpa}. HPA* reuses precomputed local distances, whereas ATOM reuses lower-level option policies and their state-dependent cost-to-go, while refining abstract transition costs from observed option initiation states.

Hierarchical MDP methods decompose sequential decision problems into subtasks. The options framework \cite{sutton} formalizes temporally extended actions, while MAXQ provides a hierarchical value-function decomposition \cite{maxq}. Guan et al. \cite{scott} use options for hierarchical problem decomposition, while Konidaris et al. \cite{skills_symbols} connect high-level actions to abstract representations for planning. Larsson et al. \cite{hier_state_abs} consider computational effort when selecting among abstraction levels, and Rodriguez-Sanchez and Konidaris \cite{rodriguez} learn abstract world models for planning with options. ATOM operates on a given hierarchical representation and computes each lower-level option policy once. Navigation experience then refines its abstract transition cost by reweighting the stored state-dependent cost-to-go according to the smoothed empirical initiation-state distribution. This allows ATOM to reuse lower-level planning solutions across navigation tasks while adapting higher-level transition costs from experience, avoiding repeated option-policy computation as the abstract model is refined.

\section{Methodology}

\subsection{Problem Formulation and Framework Overview}

ATOM considers repeated navigation over a hierarchical graph consisting of a higher-level graph $\mathcal{G}_R=(\mathcal{R},\mathcal{E}_R)$ and a lower-level graph $\mathcal{G}_P=(\mathcal{S},\mathcal{E}_P)$. 
Each abstract \textit{region} $R_i\in\mathcal{R}$ contains lower-level \textit{states} $S_i\subset\mathcal{S}$, 
while $\{R_i,R_j\}\in\mathcal{E}_R$ indicates connectivity between adjacent regions. Given a start-goal query $(s_{\mathrm{s}},s_{\mathrm{g}})$, let $R_{\mathrm{s}}$ and $R_{\mathrm{g}}$ denote the regions containing $s_{\mathrm{s}}$ and $s_{\mathrm{g}}$, respectively. ATOM reuses previously computed local planning solutions to avoid recomputing recurring local navigation subproblems across queries.

For each query, the start and goal states are mapped to their corresponding regions, and A* searches for a sequence of abstract transitions. Each traversal $R_i\rightarrow R_j$ is represented by a directed option $o_{ij}$. 
When $o_{ij}$ is first encountered, its 
policy and state-dependent cost-to-go are computed and stored. Subsequent executions reuse this solution and refine the abstract transition cost from observed initiation states.

\subsection{Hierarchical Graph Representation}

ATOM is applicable to hierarchical graph representations satisfying the structure above. In this work, as a concrete example, $\mathcal{G}_R$ and $\mathcal{G}_P$ are instantiated by the \textsf{Rooms} and \textsf{Places} layers, respectively, of a 3DSG generated by Hydra \cite{hydra}. Each \textsf{Room} $R_i$ is associated with $S_i$, with Room and Places edges representing regional connectivity and lower-level traversability, respectively. We assume that the hierarchical graph topology and lower-level transition model remain fixed during repeated navigation, allowing the option policies to be computed once and reused. The abstract transition costs, however, are refined from navigation experience as described in Section~\ref{sec:transition_refinement}.

\subsection{Transition Options}

ATOM associates each abstract adjacency $\{R_i,R_j\}\in\mathcal{E}_R$ with directed navigation options. Although $\mathcal{G}_R$ is undirected, opposite traversals define distinct options $o_{ij}$ and $o_{ji}$, allowing each direction to maintain its own policy, cost-to-go, and experience. For $R_i\rightarrow R_j$, the initiation set is $\mathcal{I}_{ij}=S_i$.
The terminal set containing states in $R_j$ directly connected to $R_i$ is $\mathcal{F}_{ij}=
\left\{s'\in S_j \mid \exists s\in S_i: (s,s')\in\mathcal{E}_P \right\}$. Following the options framework \cite{sutton}, termination is defined by
\begin{equation}
    \beta_{ij}(s)=
    \begin{cases}
        1, & s\in\mathcal{F}_{ij},\\
        0, & \text{otherwise}.
    \end{cases}
    \label{eq:termination}
\end{equation}
The directed option is defined as $o_{ij}=\left\langle\mathcal{I}_{ij},\pi_{ij},\beta_{ij}\right\rangle$, where $\pi_{ij}$ drives the agent from $\mathcal{I}_{ij}$ to $\mathcal{F}_{ij}$. Once selected, $\pi_{ij}$ executes until $\mathcal{F}_{ij}$ is reached.

\subsection{Local Option Policy Generation}

When $o_{ij}$ is encountered for the first time, ATOM constructs a local MDP $\mathcal{M}_{ij}=\left\langle\mathcal{X}_{ij},\mathcal{A}_{ij},P_{ij},c_{ij},\mathcal{F}_{ij}\right\rangle$, with state space
$\mathcal{X}_{ij}=S_i \cup \mathcal{F}_{ij}$, where $\mathcal{A}_{ij}(s)$ contains traversable lower-level edges from $s$, and $P_{ij}(s'|s,a)$ is the deterministic transition model. 
Each transition incurs a Euclidean traversal cost $c_{ij}(s,a,s')=\left\|\mathbf{p}(s')-\mathbf{p}(s)\right\|_2$, where $\mathbf{p}(s)$ is the spatial position of $s$. 
Terminal states $f\in\mathcal{F}_{ij}$ are absorbing.

We use value iteration to solve $\mathcal{M}_{ij}$ to obtain the policy, $\pi_{ij}$, and the corresponding cost-to-go
\begin{equation}
    J_{ij}(s)
    =
    \min_{\pi}
    \mathbb{E}_{\pi}
    \left[
        \sum_{k=0}^{\tau_{ij}-1}
        c_{ij}(s_k,a_k,s_{k+1})
        \,\middle|\,
        s_0=s
    \right],
    \label{eq:cost_to_go}
\end{equation}
where $\tau_{ij}$ is the first hitting time of $\mathcal{F}_{ij}$, with $J_{ij}(f)=0$ for all $f\in\mathcal{F}_{ij}$.

After $\mathcal{M}_{ij}$ is solved, $\pi_{ij}$ and $J_{ij}$ remain fixed and are stored for reuse. 
For this purpose,
ATOM maintains an option library $\mathcal{O}$; once $o_{ij}$ is generated, $\pi_{ij}$ and $J_{ij}$ are reused whenever the same directed transition is selected again, avoiding repeated solution of $\mathcal{M}_{ij}$.

\subsection{Experience-Informed Transition Refinement}
\label{sec:transition_refinement}

The cost of an abstract transition depends on the state from which its option is initiated. Rather than assuming all initiation states are equally likely, ATOM estimates an option-specific initiation-state distribution from repeated executions. Let $n_{ij}^{(t)}(s)$ denote the number of times $o_{ij}$ has been initiated from $s\in S_i$ after $t$ executions. With symmetric Dirichlet pseudocount $\alpha>0$, the smoothed empirical distribution is
\begin{equation}
    \mu_{ij}^{(t)}(s)
    =
    \frac{
        \alpha+n_{ij}^{(t)}(s)
    }{
        \alpha|S_i|
        +
        \sum_{x\in S_i}n_{ij}^{(t)}(x)
    }.
    \label{eq:init_dist}
\end{equation}
We use $\alpha=1$, yielding a uniform prior before any executions. The experience-informed abstract transition cost is then
\begin{equation}
    C_{ij}^{(t)}
    =
    \sum_{s\in S_i}
    \mu_{ij}^{(t)}(s)J_{ij}(s).
    \label{eq:abstract_cost}
\end{equation}
When $o_{ij}$ is initiated from $s_{\mathrm{start}}$, its count is updated as
\begin{equation}
    n_{ij}^{(t+1)}(s)
    =
    n_{ij}^{(t)}(s)
    +
    \mathbbm{1}\{s=s_{\mathrm{start}}\},
    \quad s\in S_i,
    \label{eq:count_update}
\end{equation}
after which $\mu_{ij}$ and $C_{ij}$ are recomputed while $\pi_{ij}$ and $J_{ij}$ remain fixed. Thus, experience refines the abstract transition model by reweighting stored state-dependent costs rather than resolving the local MDP.

Before $o_{ij}$ is generated, its abstract edge cost is initialized by the Euclidean distance between room centroids. Once $J_{ij}$ becomes available, this is replaced by $C_{ij}^{(0)}$ and subsequently refined through option executions. Updated costs are used by A* for subsequent queries.

\subsection{Repeated Navigation}

Algorithm~\ref{alg:atom} summarizes the proposed repeated navigation sequence. For each query, A* searches $\mathcal{G}_R$ using the current directed transition costs and a Euclidean room-centroid heuristic. Missing options along the selected route are generated concurrently, while previously generated options reuse their stored $\pi_{ij}$ and $J_{ij}$. Options are then executed sequentially, with $n_{ij}$, $\mu_{ij}$, and $C_{ij}$ updated after each execution for use in subsequent queries; the current abstract route is not replanned.

After the final option reaches the goal region $R_{\mathrm{g}}$, A* is applied to the \textsf{Places} subgraph within $R_{\mathrm{g}}$ to connect the current state to the query-specific goal $s_{\mathrm{g}}$. If $s_{\mathrm{s}}$ and $s_{\mathrm{g}}$ lie in the same region, this local A* search is applied directly without generating or executing inter-region options.

\begin{algorithm}[t]
\caption{ATOM Repeated Navigation}
\label{alg:atom}
\begin{algorithmic}[1]
\Require $\mathcal{G}_R,\mathcal{G}_P$, queries $\mathcal{Q}=\{(s_{\mathrm{s}}^q,s_{\mathrm{g}}^q)\}_{q=1}^{Q}$
\State $\mathcal{O}\gets\emptyset$; initialize $C_{ij}$ by room-centroid distance
\For{$q=1,\ldots,Q$}
    \State Map $(s_{\mathrm{s}}^q,s_{\mathrm{g}}^q)$ to $(R_{\mathrm{s}},R_{\mathrm{g}})$
    \State $P_R\gets\textsc{A*}(\mathcal{G}_R,R_{\mathrm{s}},R_{\mathrm{g}},\{C_{ij}\})$
    \State $\mathcal{U}\gets\{o_{ij}\mid(R_i\rightarrow R_j)\in P_R,\;o_{ij}\notin\mathcal{O}\}$
    \ForAll{$o_{ij}\in\mathcal{U}$ \textbf{in parallel}}
        \State Compute $\mathcal{F}_{ij}$; construct and solve $\mathcal{M}_{ij}$
        \State Store $\pi_{ij},J_{ij}$; initialize $n_{ij}(s)\gets0$
        \State Compute $\mu_{ij}$ and $C_{ij}$ using \eqref{eq:init_dist}--\eqref{eq:abstract_cost}
        \State $\mathcal{O}\gets\mathcal{O}\cup\{o_{ij}\}$
    \EndFor
    \For{each $R_i\rightarrow R_j$ in $P_R$}
        \State $s_{\mathrm{start}}\gets$ current lower-level state
        \State Execute $\pi_{ij}$ until reaching $\mathcal{F}_{ij}$
        \State $n_{ij}(s_{\mathrm{start}})\gets n_{ij}(s_{\mathrm{start}})+1$
        \State Update $\mu_{ij}$ and $C_{ij}$ using \eqref{eq:init_dist}--\eqref{eq:abstract_cost}
    \EndFor
    \State Compute and execute an A* path within $R_{\mathrm{g}}$ to $s_{\mathrm{g}}^q$
\EndFor
\end{algorithmic}
\end{algorithm}

\section{Theoretical Properties}

We establish three properties of ATOM relevant to repeated navigation. First, the number of option MDPs that must be solved is bounded by the number of directed abstract transitions, independently of the number of navigation queries. 
Second, experience-informed refinement keeps each abstract transition cost within the range of its stored state-dependent cost-to-go. Finally, if the long-run empirical initiation-state frequencies stabilize, the refined transition cost converges to the expected cost-to-go under this limiting distribution.

\begin{proposition}[Bounded option generation]
For a fixed hierarchical graph, each directed option MDP is solved at most once. Hence, across any number of repeated navigation queries,
\begin{equation}
    N_{\mathrm{MDP}}^{\mathrm{opt}}
    =
    D
    \leq
    2|\mathcal{E}_R|,
    \label{eq:num_mdp_solves}
\end{equation}
where $D$ is the number of distinct directed options encountered.
\label{prop:reuse}
\end{proposition}

\begin{proof}
Each undirected edge $\{R_i,R_j\}\in\mathcal{E}_R$ induces at most two directed options, $o_{ij}$ and $o_{ji}$. Once generated, an option reuses its stored $\pi_{ij}$ and $J_{ij}$ without resolving $\mathcal{M}_{ij}$.
\end{proof}

Proposition~\ref{prop:reuse} shows that option-generation effort is bounded by the abstract graph topology rather than the number of repeated navigation queries.

\begin{proposition}[Bounded transition-cost refinement]
For every directed option $o_{ij}$ and execution count $t$,
\begin{equation}
    \min_{s\in S_i}J_{ij}(s)
    \leq
    C_{ij}^{(t)}
    \leq
    \max_{s\in S_i}J_{ij}(s).
    \label{eq:cost_bounds}
\end{equation}
\label{prop:bounded_cost}
\end{proposition}

\begin{proof}
Since $\mu_{ij}^{(t)}$ is a probability distribution over $S_i$, $C_{ij}^{(t)}=\sum_s\mu_{ij}^{(t)}(s)J_{ij}(s)$ is a convex combination of the stored values $J_{ij}(s)$.
\end{proof}

Thus, refinement reweights the stored local costs without producing an abstract transition cost outside their range.

The previous two propositions
characterize the bounded computational effort and stability of ATOM's reuse mechanism. To characterize the long-term behavior of the experience-informed transition-cost refinement, we additionally require that the initiation pattern of a repeatedly executed option eventually stabilizes.

\begin{assumption}[Stabilizing initiation frequencies]
For each directed option $o_{ij}$,
the empirical initiation-state frequencies converge, i.e.,
\begin{equation}
    \frac{n_{ij}^{(t)}(s)}{t}
    \rightarrow q_{ij}(s),
    \qquad
    \forall s\in S_i,
    \label{eq:empirical_frequency_convergence}
\end{equation}
where $t$ is the number of executions of $o_{ij}$ and $q_{ij}$ is the limiting initiation-state distribution.
\label{assump:freq_convergence}
\end{assumption}

Assumption~\ref{assump:freq_convergence} represents repeated-navigation settings in which the long-run pattern of option initiation stabilizes. This condition holds, for example, when initiation states are sampled from a stationary distribution.

\begin{proposition}[Consistency of transition-cost refinement]
\label{prop:convergence}
Under Assumption~\ref{assump:freq_convergence}, for any fixed $\alpha>0$,
\begin{equation}      \label{eq:cost_convergence}
    \mu_{ij}^{(t)}(s)
    \rightarrow
    q_{ij}(s),
    \qquad
    C_{ij}^{(t)}
    \rightarrow
    \sum_{s\in S_i}
    q_{ij}(s)J_{ij}(s),
\end{equation}
where $q_{ij}$ as in Assumption~\ref{assump:freq_convergence}.
\end{proposition}

\begin{proof}
Since $\sum_{x\in S_i}n_{ij}^{(t)}(x)=t$,
\begin{equation}
    \mu_{ij}^{(t)}(s)
    =
    \frac{
        \alpha/t+n_{ij}^{(t)}(s)/t
    }{
        \alpha|S_i|/t+1
    }.
\end{equation}
As $t\rightarrow\infty$, the prior terms vanish and Assumption~\ref{assump:freq_convergence} yields $\mu_{ij}^{(t)}(s)\rightarrow q_{ij}(s)$. 
Since $S_i$ is finite and $J_{ij}$ remains fixed, substituting this limit into \eqref{eq:abstract_cost} yields \eqref{eq:cost_convergence}.
\end{proof}

Proposition~\ref{prop:convergence} shows that, when the long-run initiation pattern stabilizes, the influence of the initial smoothing prior vanishes and the refined abstract transition cost converges to the expected stored cost-to-go under the limiting initiation-state distribution. Thus, with repeated experience, the abstract cost becomes representative of where the option is actually initiated rather than of the initial uniform estimate.

\section{Experiments and Results}
\label{sec:experiments}

\subsection{Experimental Setup}

We evaluated ATOM through repeated-navigation simulations, a controlled transition-refinement experiment, and physical-robot experiments. The main simulation uses the uHumans2 office environment \cite{uHumans2}, with a hierarchical 3DSG generated by Hydra \cite{hydra}, shown in Figure~\ref{fig:uhumans2_sg}, containing 9 \textsf{Rooms}, 492 \textsf{Places}, and 9 undirected Room edges (18 possible directed options). Experiments run in Python on Ubuntu 24.04 with an Intel Core i5-14600K CPU, NVIDIA RTX 4060 Ti 8 GB GPU, and 32 GB RAM. Value iteration uses synchronous Bellman updates with tolerance $10^{-9}$ and a maximum of $10^4$ iterations.

\begin{figure}[t]
    \centering
    \includegraphics[width=\columnwidth]{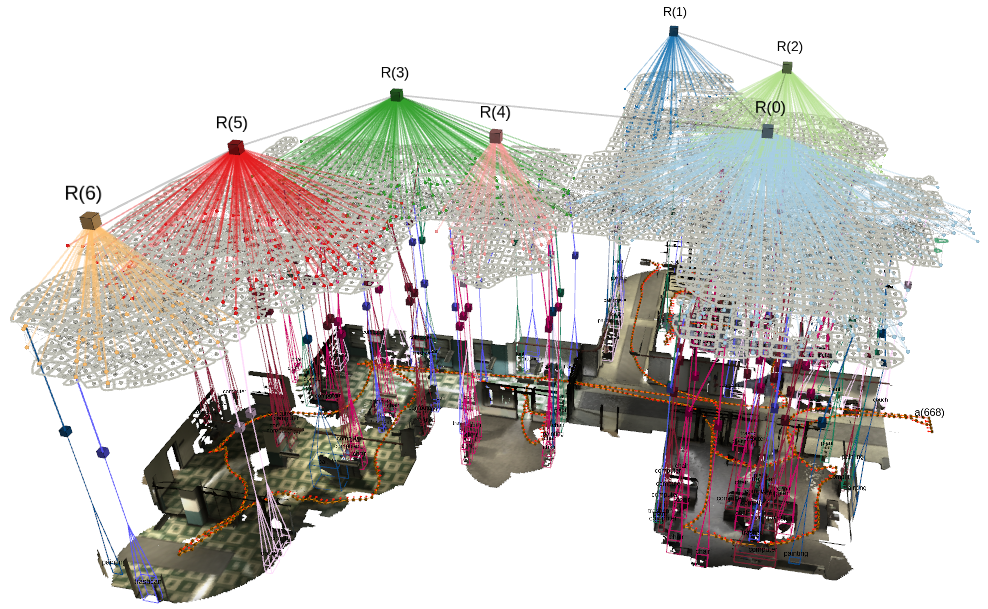}
    \caption{Hierarchical 3D scene graph of the uHumans2 \cite{uHumans2} office environment generated using Hydra \cite{hydra}. The \textsf{Rooms} and \textsf{Places} layers provide the higher- and lower-level graph representations used by ATOM for repeated navigation.}
    \label{fig:uhumans2_sg}
\end{figure}

We compare \emph{Global Places A*}, which performs shortest-path search over the complete lower-level graph for each query; \emph{Full Graph MDP}, which independently solves the complete lower-level MDP for each query; \emph{ATOM-NoReuse}, which regenerates required local policies for each query; \emph{ATOM-Fixed}, which reuses options with fixed costs $C_{ij}^{(0)}$; and full \emph{ATOM}, which additionally refines $C_{ij}^{(t)}$. ATOM variants solve independent local MDPs using up to four processors, while Full Graph MDP solves one monolithic value-iteration problem.

We performed 10 trials of 500 randomly sampled valid start-goal queries, using identical query sequences across methods and resetting the option library between trials. Planning time includes graph/MDP preparation, higher-level search, value iteration, option generation or retrieval, transition-cost updates, final goal-region planning, and path extraction; graph loading, query generation, logging, plotting, and execution are excluded.
The results are shown in Table~\ref{tab:simulation_results} and Figure~\ref{fig:runtime}.

We additionally report the number of MDP solves, solved-state volume, and trajectory cost. Solved-state volume sums the state counts of all MDPs actually solved, so reused options incur no additional volume. For trajectory quality, Global Places A* is used to obtain the shortest-path reference, and the relative trajectory-cost increase is evaluated as
\begin{equation}
    \Delta L =
    \frac{L-L_{\mathrm{A^*}}}{L_{\mathrm{A^*}}}\times100\%.
    \label{eq:trajectory_gap}
\end{equation}
With deterministic Euclidean costs, Global Places A* and Full Graph MDP obtain the same minimum trajectory cost up to floating-point precision.

\begin{figure}[t]
    \centering
    \includegraphics[width=\columnwidth]{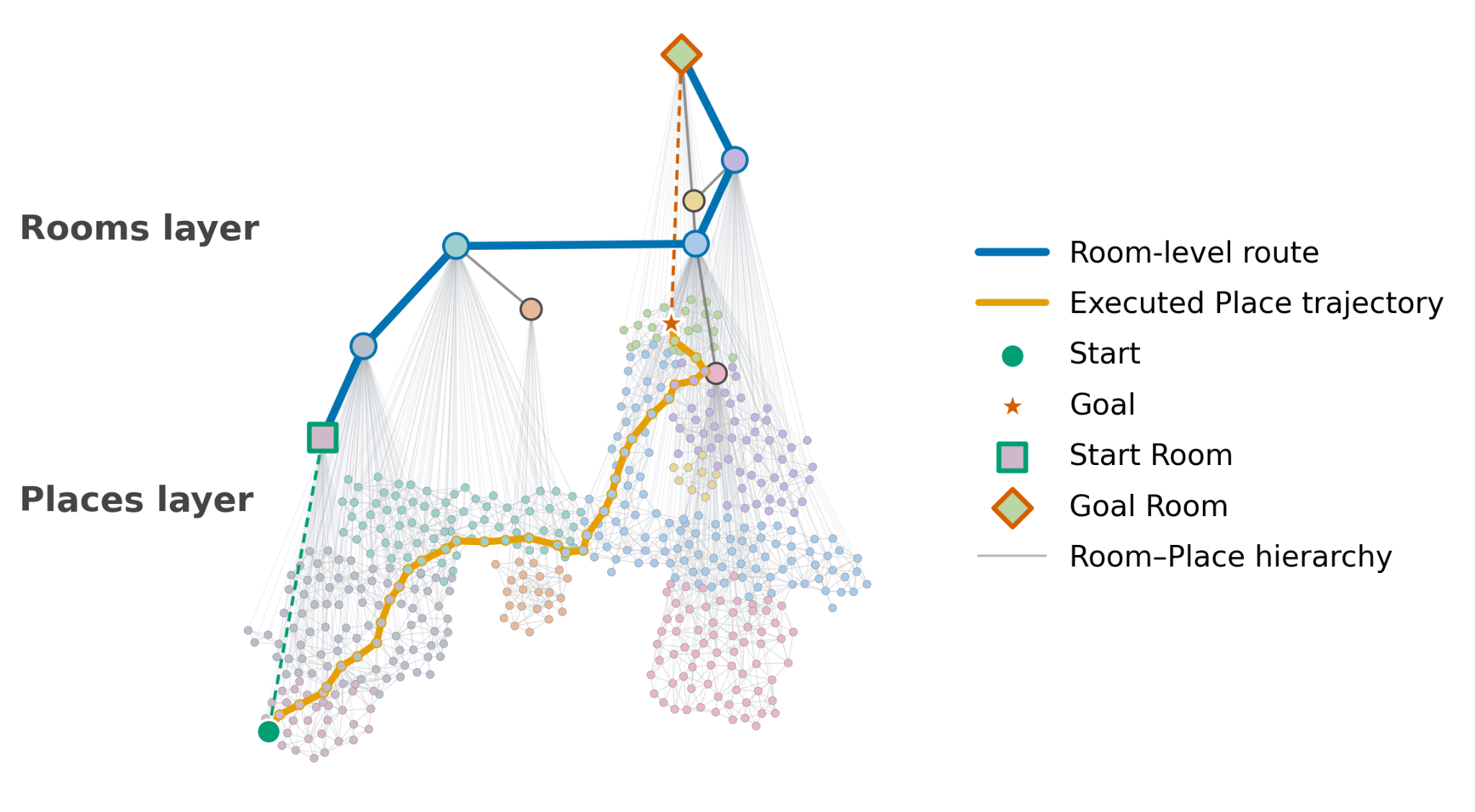}
    \caption{Representative ATOM query in uHumans2. 
    The \textsf{Rooms} layer shows the abstract route and the \textsf{Places} layer the executed trajectory.}
    \label{fig:uhumans_trajectory}
\end{figure}

\subsection{Computational Efficiency and Option Reuse}

Figure~\ref{fig:uhumans_trajectory} illustrates ATOM's hierarchical decomposition. Table~\ref{tab:simulation_results} summarizes repeated-navigation performance, with Full Graph MDP providing the computational reference for repeated policy generation.

\begin{table}[t]
\centering
\caption{Performance over 10 trials of 500 queries (mean $\pm$ sample standard deviation).}
\label{tab:simulation_results}
\resizebox{\columnwidth}{!}{
\begin{tabular}{lcccc}
\hline
\textbf{Method} & \textbf{Time (s)} & \textbf{MDP Solves} & \textbf{Solved-State Volume} & \boldmath$\Delta L$ \textbf{(\%)} \\
\hline
Global Places A* & $0.144\!\pm\!0.009$ & -- & -- & $0.000\!\pm\!0.000$ \\
Full Graph MDP & $4.955\!\pm\!0.077$ & $500.0\!\pm\!0.0$ & $246000\!\pm\!0$ & $0.000\!\pm\!0.000$ \\
ATOM-NoReuse & $1.945\!\pm\!0.056$ & $842.8\!\pm\!18.8$ & $77223\!\pm\!1512$ & $7.504\!\pm\!0.998$ \\
ATOM-Fixed & $0.102\!\pm\!0.008$ & $17.1\!\pm\!0.3$ & $1318\!\pm\!9$ & $7.504\!\pm\!0.998$ \\
ATOM & $0.122\!\pm\!0.010$ & $17.1\!\pm\!0.3$ & $1318\!\pm\!9$ & $7.303\!\pm\!0.952$ \\
\hline
\end{tabular}}
\end{table}

\begin{figure}[t]
    \centering
    \includegraphics[width=\columnwidth]{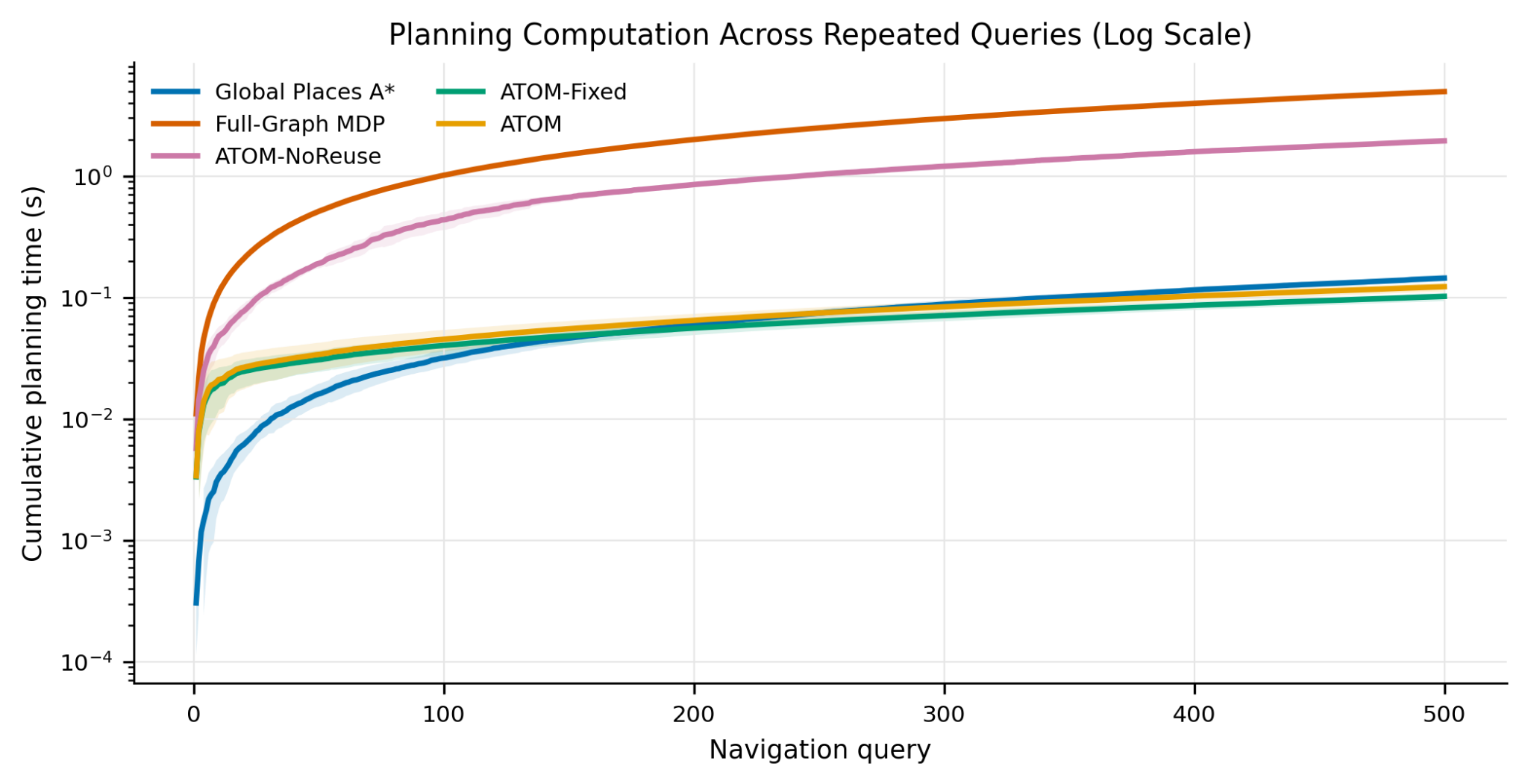}
    \caption{Cumulative planning time over 500 queries on a logarithmic scale (mean $\pm$ one sample standard deviation across 10 trials).}
    \label{fig:runtime}
\end{figure}

ATOM reduces cumulative planning time by $97.5\%$ and solved-state volume by $99.5\%$ relative to Full Graph MDP. Relative to ATOM-NoReuse, option reuse reduces planning time by $93.7\%$ and solved-state volume by $98.3\%$. Across 500 queries, only $17.1$ distinct inter-Room option MDPs are generated on average, while $98.0\%$ of executed inter-Room options reuse stored solutions. Because the query-specific connection to the exact goal is computed using A* within the final region, no additional goal-conditioned MDP is required for each query. ATOM-Fixed and ATOM require $0.102$ and $0.122$~s of cumulative planning time, respectively, showing that experience-informed transition-cost refinement adds little computational overhead.

As shown in Figure~\ref{fig:runtime}, Full Graph MDP and ATOM-NoReuse accumulate substantially greater computation, whereas option reuse slows the growth of cumulative planning time for ATOM-Fixed and ATOM. ATOM requires $0.122\pm0.010$~s over 500 queries, compared with the $0.144\pm0.009$~s required by Global Places A*. Unlike Global Places A*, which performs a new global search for each query, ATOM amortizes option-generation cost through reuse, resulting in lower incremental planning cost once reusable options have been generated. This computational efficiency comes with a modest trajectory-cost increase relative to the Global Places A* shortest-path reference: ATOM produces trajectories $7.30\%\pm0.95\%$ above the reference, compared with $7.50\%\pm1.00\%$ for both ATOM-NoReuse and ATOM-Fixed.

\subsection{Experience-Informed Transition Refinement}

We evaluate whether option experience improves abstract transition-cost estimates using the uHumans2 repeated-navigation experiments described above. For each option--trial pair executed at least 20 times, the final $30\%$ of executions provide an empirical reference from realized initiation states, while the preceding $70\%$ reconstruct the evolving estimator. This yields 126 eligible option--trial pairs and a common horizon of 14 observed executions.

\begin{figure}[t]
    \centering
    \includegraphics[width=0.85\columnwidth]{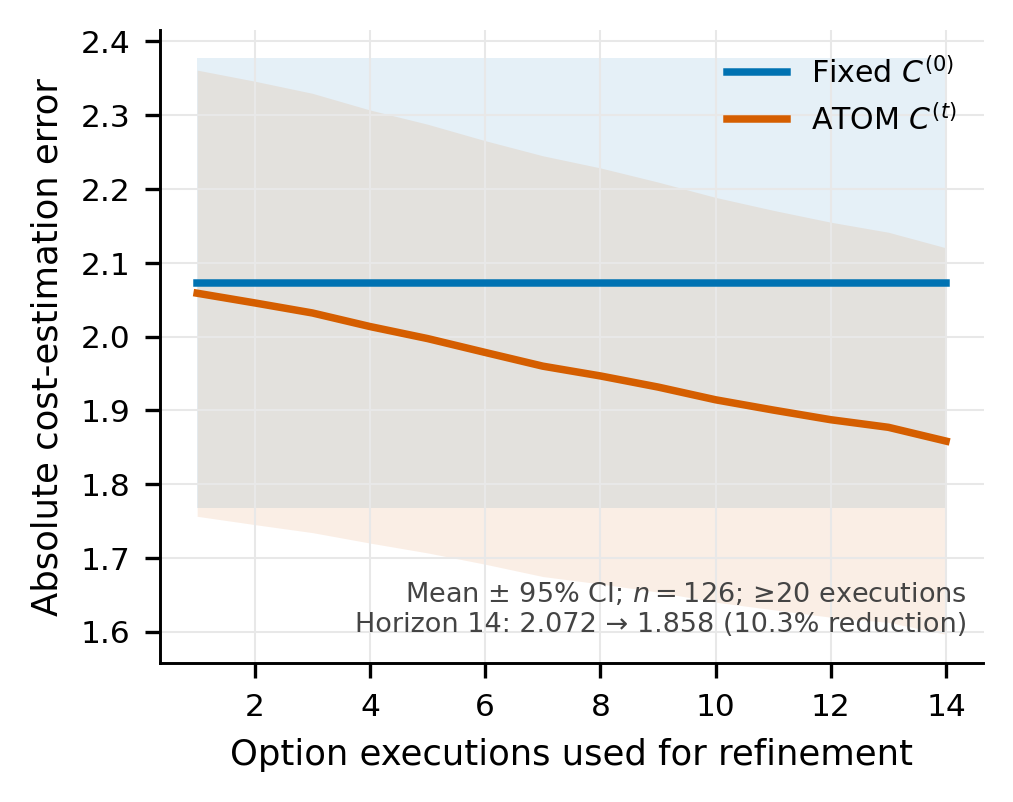}
    \caption{Transition-cost estimation error as option experience accumulates (mean $\pm$ 95\% CI across 126 eligible option--trial pairs). ATOM reduces the error from $2.072$ to $1.858$ after 14 executions.}
    \label{fig:cost_estimation}
\end{figure}

As shown in Figure~\ref{fig:cost_estimation}, mean absolute error decreases from $2.072$ for the fixed estimate $C_{ij}^{(0)}$ to $1.858$ after 14 observed executions, a $10.3\%$ reduction. This improvement results solely from reweighting the fixed $J_{ij}$ using observed initiation states; the underlying option policies and cost-to-go remain unchanged.

Because the uHumans2 Room graph provides few competing abstract routes, we additionally use the four-Room graph in Figure~\ref{fig:refinement} to isolate whether refinement can alter higher-level decisions. Two routes connect $R_1$ and $R_4$, with all option models generated before repeated navigation. ATOM and ATOM-Fixed therefore differ only in transition-cost refinement.

\begin{figure}[t]
    \centering
    \begin{minipage}{\columnwidth}
        \centering
        \includegraphics[width=\columnwidth]{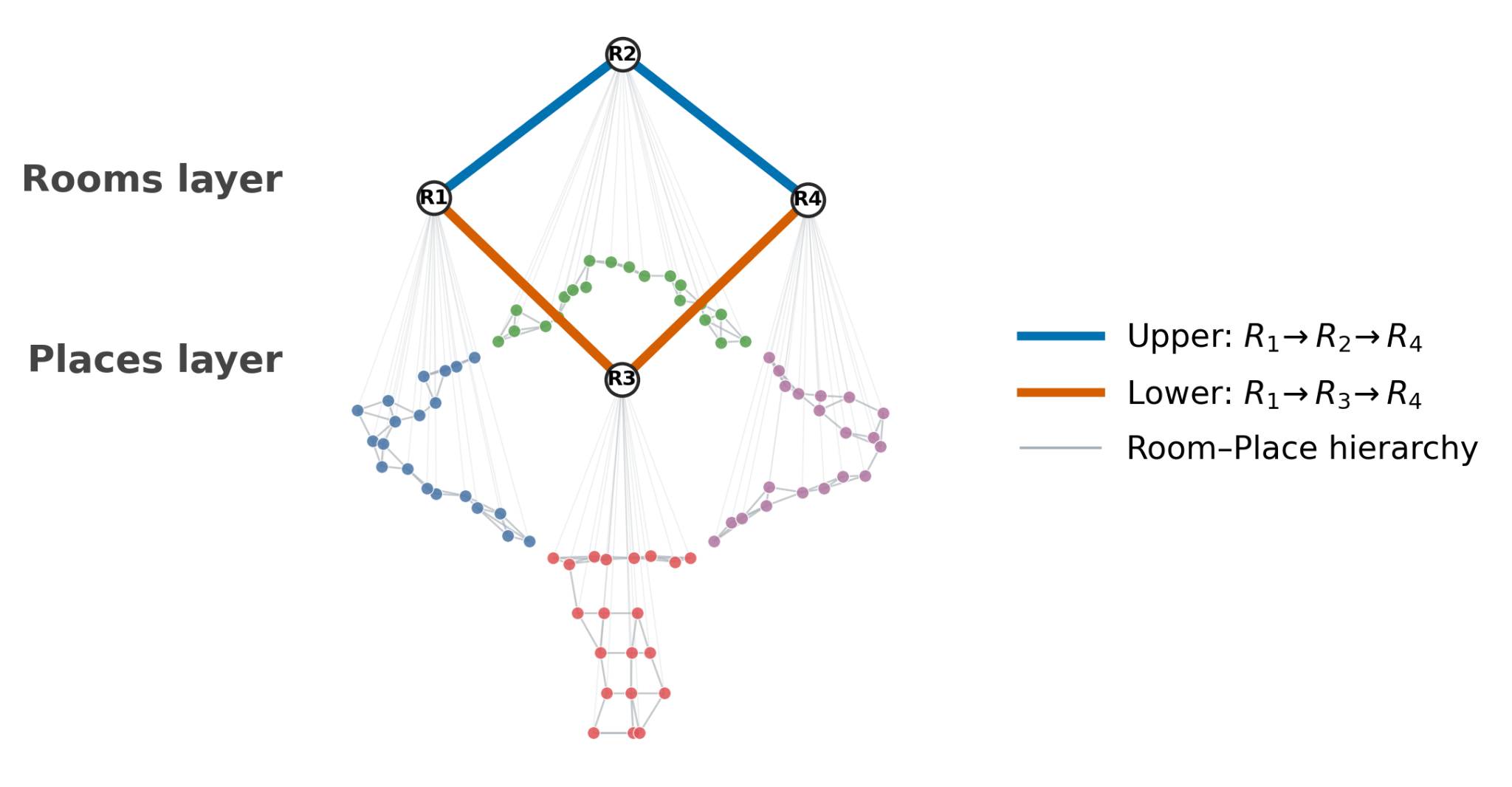}\\[-1mm]
        {\small (a)}\\[1mm]
        \includegraphics[width=\columnwidth]{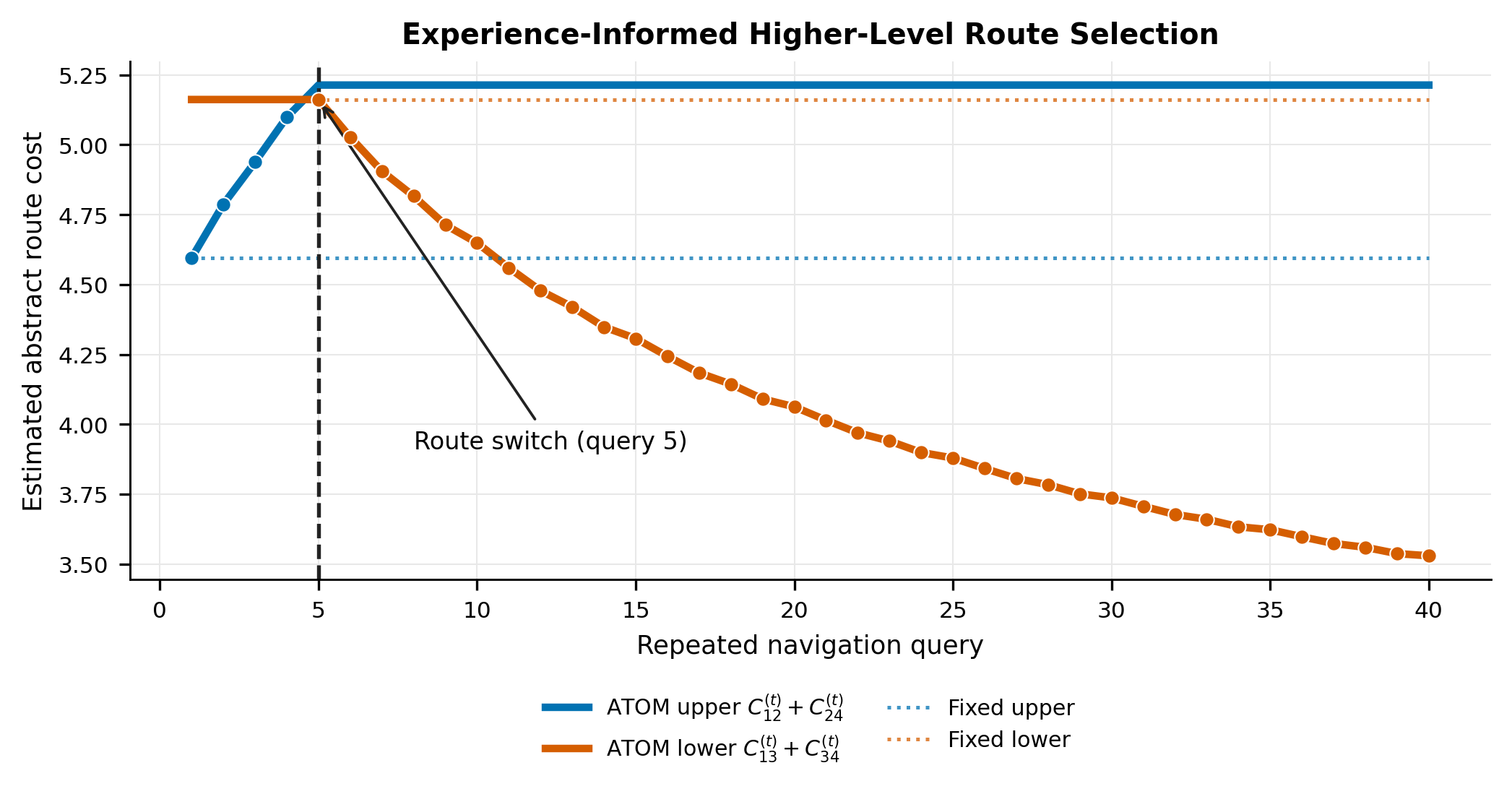}\\[-1mm]
        {\small (b)}
    \end{minipage}
    \caption{Controlled competing-route experiment. (a) Two alternative routes connect $R_1$ and $R_4$. (b) Refinement reverses their estimated cost ordering, causing ATOM to switch routes at query 5; fixed costs retain the original route.}
    \label{fig:refinement}
\end{figure}

Both methods initially select $R_1\!\rightarrow\!R_2\!\rightarrow\!R_4$, with estimated cost $4.595$ versus $5.162$ for $R_1\!\rightarrow\!R_3\!\rightarrow\!R_4$. Refinement reverses this ordering, causing ATOM to switch to the lower route at query 5 and retain it through query 40. ATOM-Fixed selects the original route for all 40 queries. This experiment confirms that refinement can change the abstract route selected on subsequent queries.

\subsection{Real-World Computational Evaluation}

We further evaluated ATOM using an F1Tenth vehicle \cite{f1tenth} equipped with an Intel RealSense D455, 2D lidar, and NVIDIA Jetson Orin Nano 8GB Developer Kit (Figure~\ref{fig:f1tenth}) on the second floor of the Van Leer Building at the Georgia Institute of Technology. Sensor data and an optimized SLAM trajectory were collected onboard, \textbf{with representative camera frames and segmentation overlays shown in Figure~\ref{fig:van_leer_segmentation},} from which the hierarchical 3DSG shown in Figure~\ref{fig:van_leer_sg} was reconstructed offline using Hydra on the same PC used for the simulation experiments and subsequently deployed to the F1Tenth for onboard ATOM planning. The graph was post-processed by treating corridor segments as abstract regions and reassigning their associated \textsf{Places} nodes. The resulting graph contains $1323$ \textsf{Places} nodes in its largest connected component and remains fixed throughout the experiments. A representative query is shown in Figure~\ref{fig:physical_graph}.

\begin{figure}[t]
    \centering
    \includegraphics[width=0.7\columnwidth]{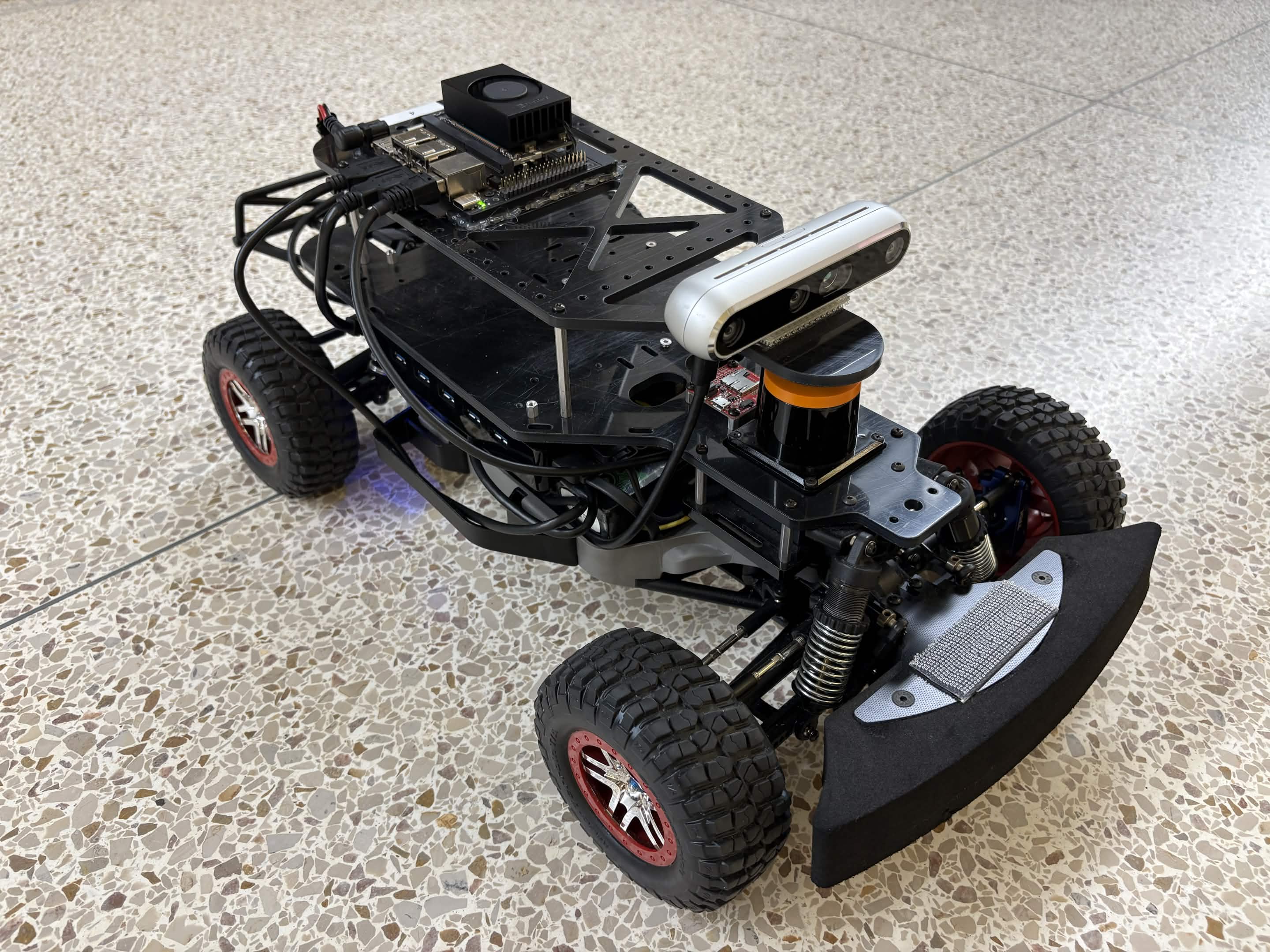}
    \caption{F1Tenth experimental platform with Intel RealSense D455 and NVIDIA Jetson Orin Nano.}
    \label{fig:f1tenth}
\end{figure}

\begin{figure}[t]
    \centering
    \includegraphics[width=\columnwidth]{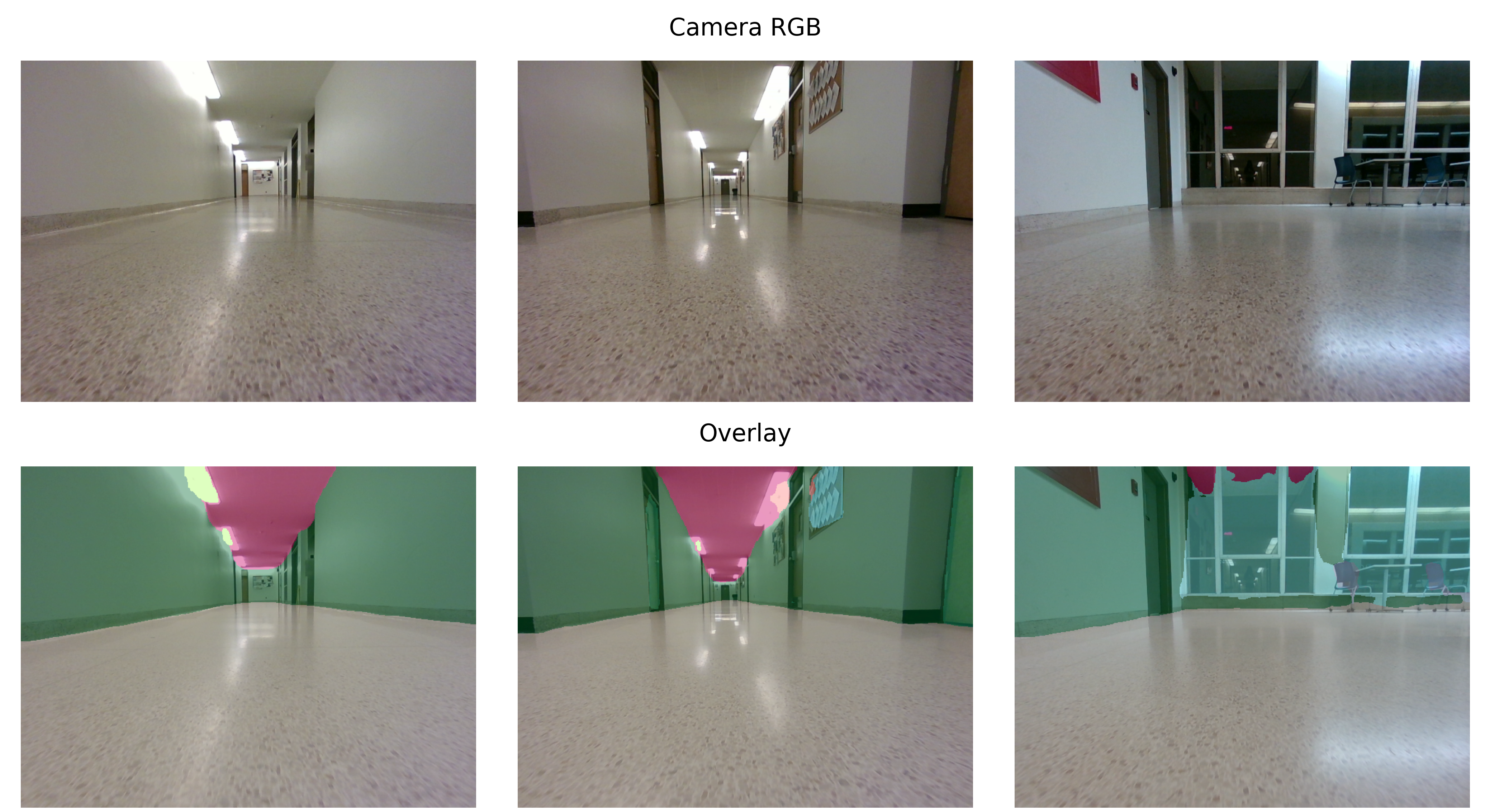}
    \caption{Representative camera observations collected in the Georgia Tech Van Leer Building. Top: RGB frames from the F1Tenth's Intel RealSense D455. Bottom: corresponding semantic segmentation overlays used during scene-graph reconstruction.}
    \label{fig:van_leer_segmentation}
\end{figure}

\begin{figure}[t]
    \centering
    \includegraphics[width=0.8\columnwidth]{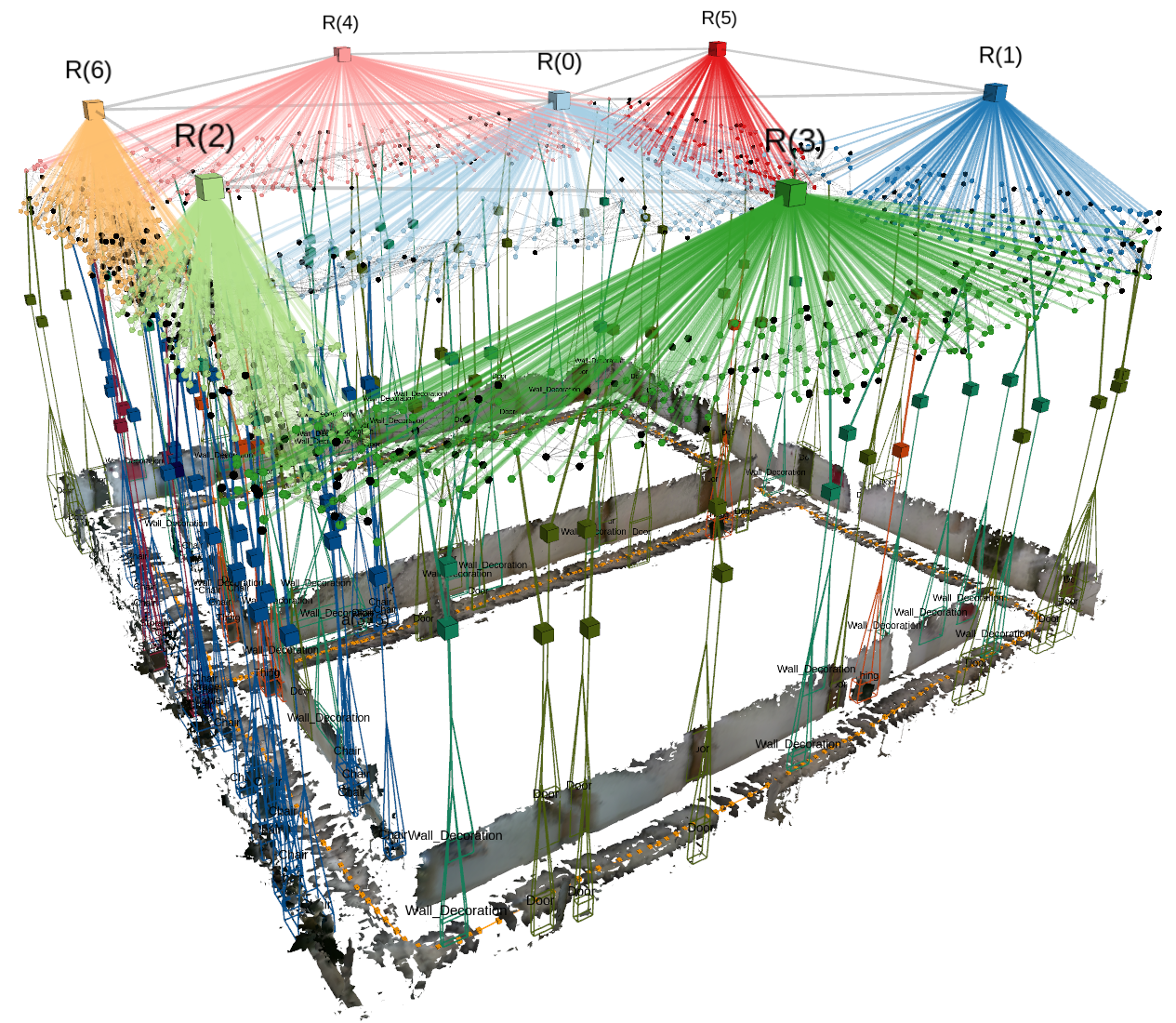}
    \caption{Hierarchical 3D scene graph of the second floor of the Georgia Tech Van Leer Building generated using Hydra. The visualization shows the reconstructed environment together with the abstract region nodes and their associated lower-level graph structure used to instantiate ATOM.}
    \label{fig:van_leer_sg}
\end{figure}

We first evaluate planning on this real-world graph using five trials of 30 random start--goal queries sampled from its largest connected \textsf{Places} component. The option library is reset between trials and retained within each trial. Full Graph MDP solves the complete $1323$-state graph independently for every query, whereas ATOM generates inter-region options as needed and reuses their stored solutions. All 150 queries were successfully evaluated. The resulting planning-time, trajectory-quality, and option-reuse statistics are summarized in Figure~\ref{fig:physical_results}.

Across the 150 queries, ATOM requires $0.496$~sec/query compared with $18.177$~sec/query for Full Graph MDP, a $97.3\%$ reduction, while reducing solved states per query from $1323$ to approximately $100$ ($92.5\%$). ATOM trajectories are on average $9.39\%$ longer than the Global Places A* reference. Across inter-region option executions, 122 reuse previously generated solutions and 77 require new option generation, corresponding to a $61.3\%$ reuse rate. The onboard results show a similar repeated-planning advantage on the larger real-world graph.

\begin{figure}[t]
    \centering
    \includegraphics[width=\columnwidth]{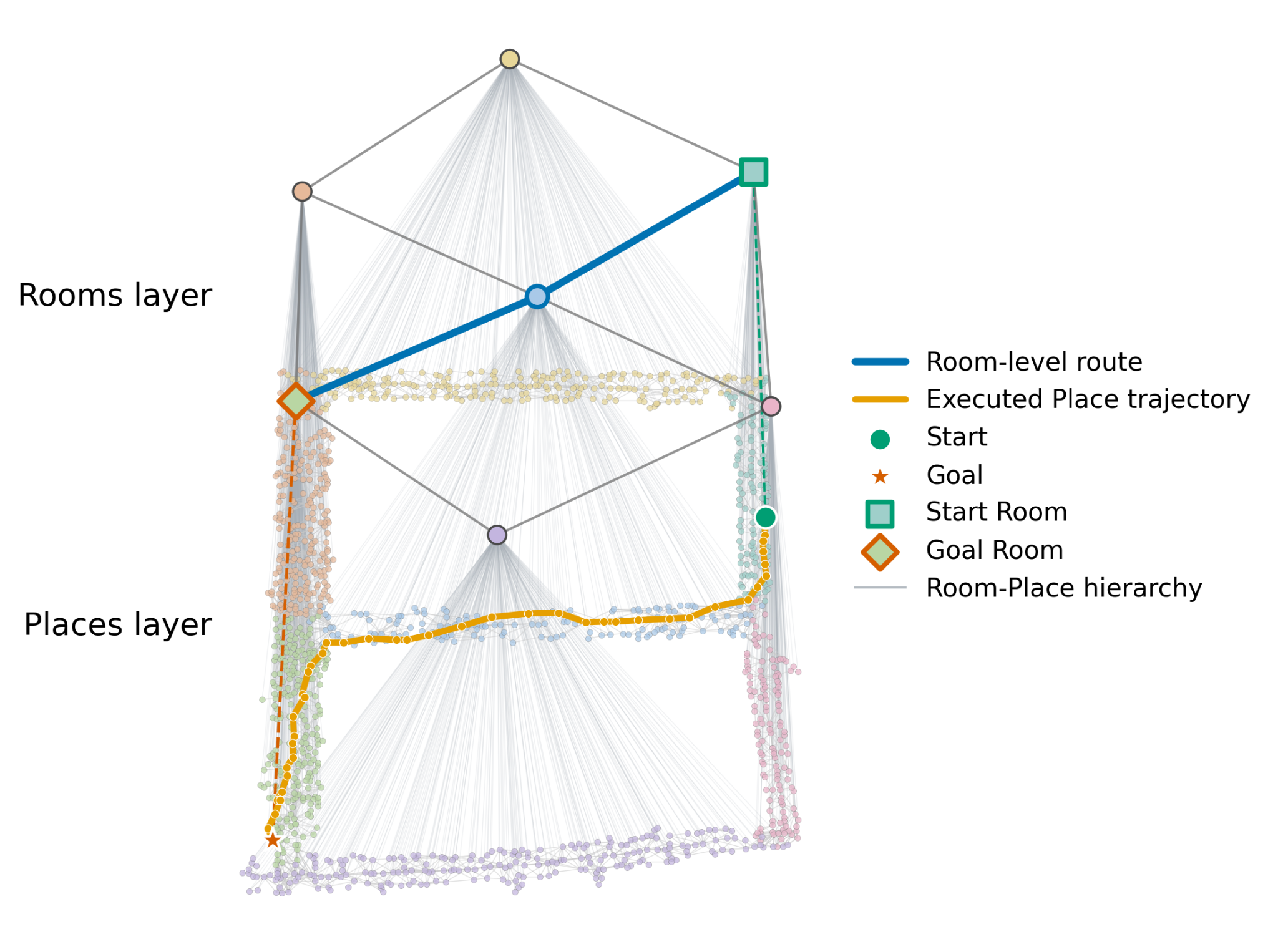}
    \caption{Representative ATOM query on the real-world graph of the second floor of the Georgia Tech Van Leer Building, showing the abstract route (blue) and lower-level trajectory (orange).}
    \label{fig:physical_graph}
\end{figure}

\begin{figure}[t]
    \centering
    \includegraphics[width=\columnwidth]{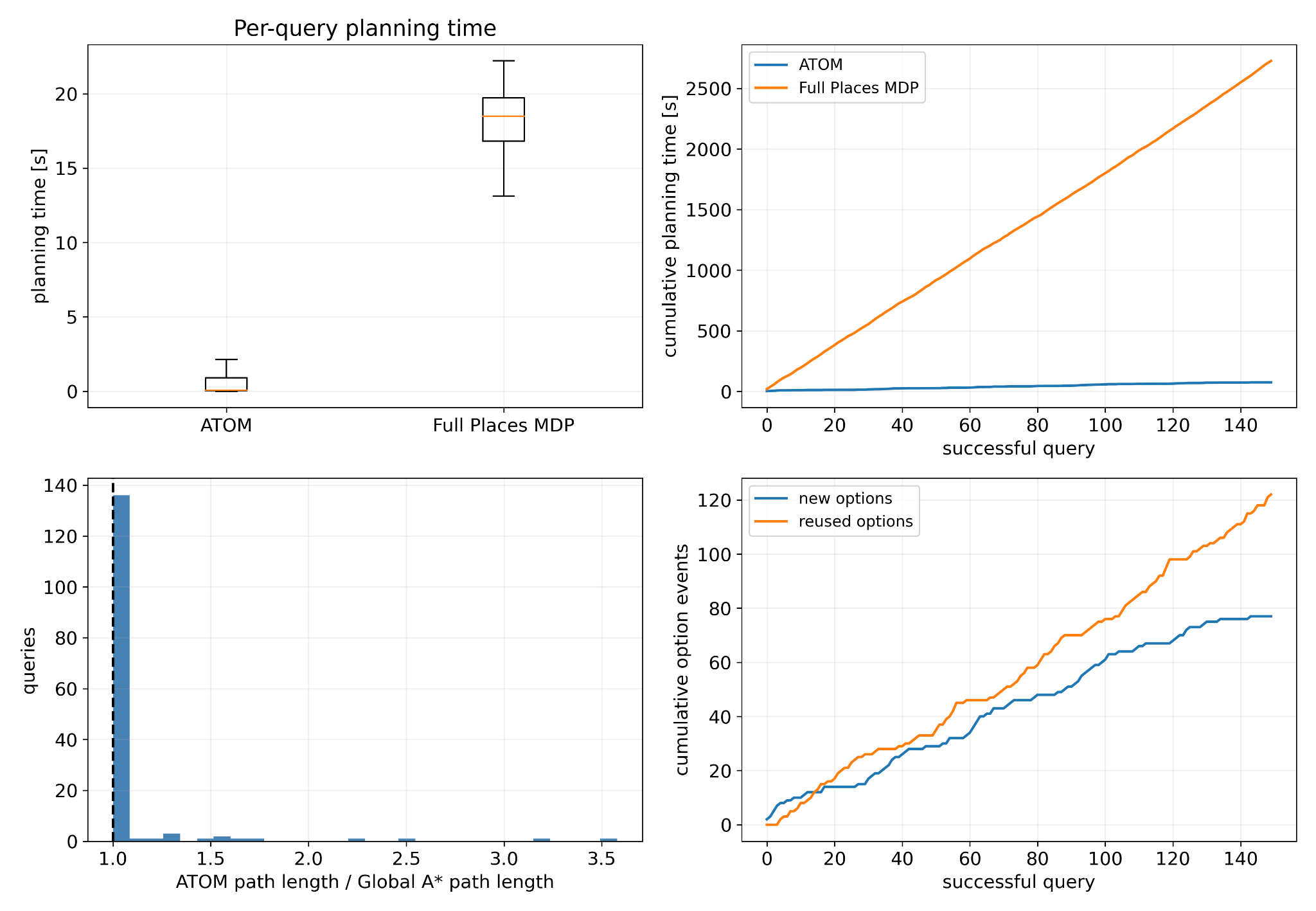}
    \caption{Onboard planning performance on the Van Leer graph over 150 successful start--goal queries. Top left: per-query planning-time distributions for ATOM and Full Graph MDP. Top right: cumulative planning time across successful queries. Bottom left: distribution of the ratio between ATOM trajectory length and the Global Places A* shortest-path reference. Bottom right: cumulative numbers of newly generated and reused inter-region options.}
    \label{fig:physical_results}
\end{figure}

\subsection{Physical Navigation Validation}

Finally, we validated ATOM during physical navigation using five sequential navigation queries selected to produce repeated traversal of the same abstract transitions. The option library is initialized empty and retained across queries, allowing options generated during earlier navigation tasks to be reused in subsequent ones. For each query, ATOM plans the abstract \textsf{Room} route and executes the associated option policies to obtain a \textsf{Place}-level reference path. To accommodate the F1Tenth's Ackermann steering constraints, this discrete path is converted into a continuous, smoothed path while preserving ATOM's selected regions and corridor transitions, with adjustments made where necessary for feasible cornering and map clearance. 
The ROS 2 Navigation Stack (Nav2) \cite{nav2}
tracks the resulting path using a regulated pure-pursuit controller, while Adaptive Monte Carlo Localization (AMCL) \cite{amcl} 
provides localization on the SLAM occupancy map.

All five navigation queries were successfully completed. Of seven inter-region option executions, three reused previously generated solutions, corresponding to a $\sim43\%$ reuse rate. Figure~\ref{fig:physical_navigation} compares representative ATOM-planned paths with the corresponding AMCL-estimated executed trajectories. These results demonstrate that ATOM's reusable option policies can generate routes that can be executed repeatedly on a physical Ackermann-steered platform.

\begin{figure}[t]
    \centering
    \includegraphics[width=\columnwidth]{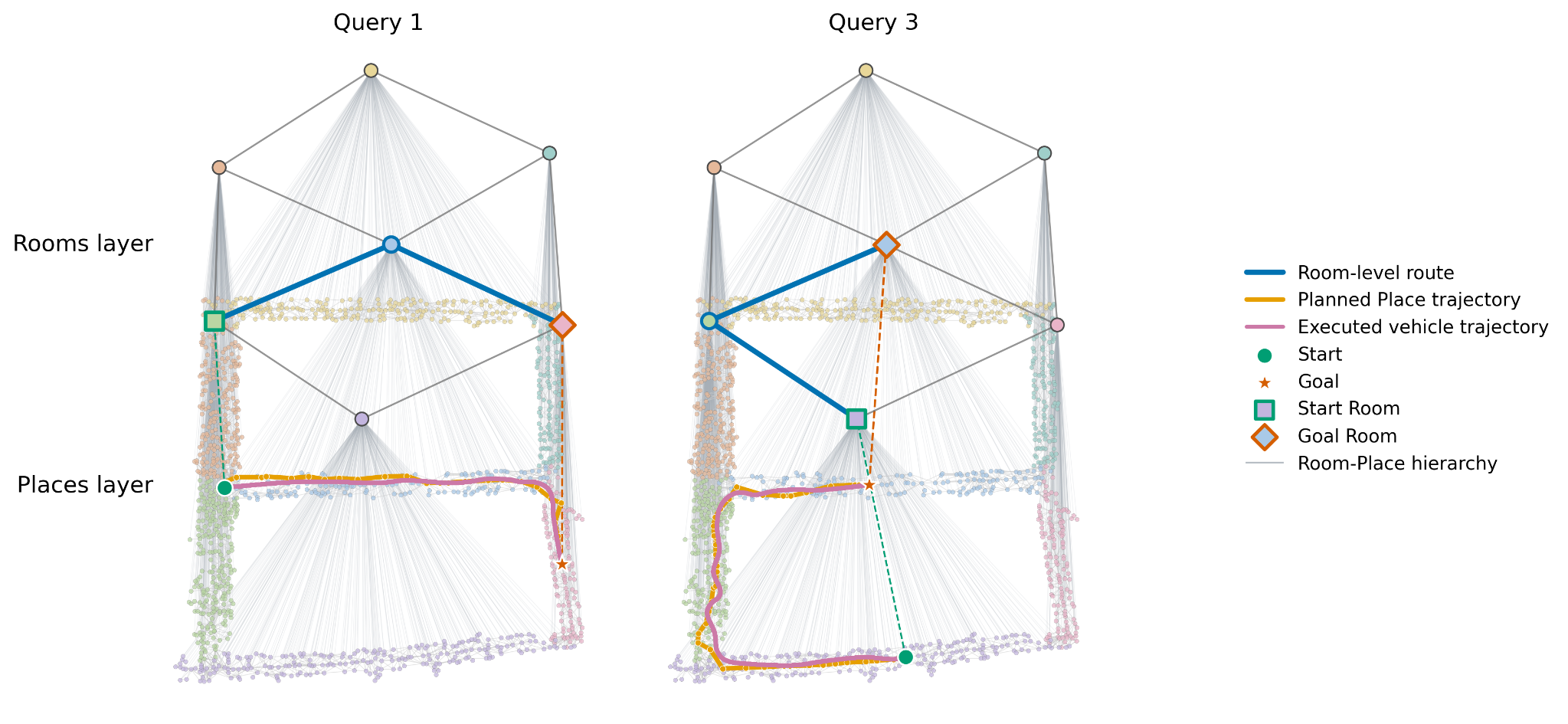}
    \caption{Physical navigation validation for representative queries. Each panel shows the abstract \textsf{Room}-level route, ATOM-planned \textsf{Place} trajectory, and executed vehicle trajectory. Query 3 includes reuse of a previously generated option.}
    \label{fig:physical_navigation}
\end{figure}

\section{Discussion and Conclusion}

We have presented ATOM, a hierarchical framework that represents directed abstract transitions as reusable navigation options. Option policies and their state-dependent cost-to-go are computed once and reused across navigation tasks, while observed initiation states progressively refine the transition costs used for subsequent higher-level planning.

In simulation, ATOM reduced cumulative planning time by $97.5\%$ relative to repeated monolithic MDP planning while producing trajectories within approximately $7\%$ of the Global Places A* shortest-path reference. Experience-informed refinement reduced held-out transition-cost estimation error by $10.3\%$ and changed the selected abstract route in the controlled competing-route experiment. On the larger $1323$-\textsf{Place} real-world graph, onboard planning reduced mean planning time by $97.3\%$ and solved-state volume by $92.5\%$ relative to Full Graph MDP. Physical navigation further validated option reuse on the F1Tenth, successfully completing all five navigation queries with a $43\%$ inter-region option reuse rate.

The present formulation assumes a fixed hierarchical graph and lower-level transition model. Future work will address changing environments through selective updates of affected option models, incorporate semantic traversal information, and integrate ATOM with evolving scene graphs.

\section*{ACKNOWLEDGMENT}

This work was supported by ARL award DCIST CRA W911NF-17-2-0181 and ONR award N00014-23-1-2304.

\bibliographystyle{IEEEtran}
\bibliography{ref}

\end{document}